\documentclass[11pt]{article}

\usepackage[final]{acl}

\usepackage{times}
\usepackage{latexsym}

\usepackage[T1]{fontenc}

\usepackage[utf8]{inputenc}

\usepackage{microtype}

\usepackage{inconsolata}

\usepackage{graphicx}

\usepackage{subcaption}
\usepackage{amsmath,amsthm,amsfonts,amssymb}
\usepackage{booktabs}
\usepackage{multirow}
\usepackage{cleveref}

\usepackage[ruled,vlined]{algorithm2e} %
\SetAlgoVlined
\SetAlgoNoEnd                    %
\LinesNumbered                   %
\SetAlgoNlRelativeSize{0}        %
\SetNlSkip{0.4em}                %
\SetAlFnt{\normalsize}           %

\usepackage[svgnames,dvipsnames]{xcolor}
\newtheorem{theorem}{Theorem}[section]

\usepackage{mathtools} %
\DeclareMathOperator{\KL}{KL}

\title{Flip-Flop Consistency:\\Unsupervised Training for Robustness to Prompt Perturbations in LLMs}

\author{Parsa Hejabi \And Elnaz Rahmati \And Alireza S. Ziabari \And Morteza Dehghani \AND
\normalfont \small  University of Southern California \\
\texttt{\small\{hejabi, erahmati, salkhord, mdehghan\}@usc.edu} \\
}

\begin{document}
\maketitle
\begin{abstract}
Large Language Models (LLMs) often produce inconsistent answers when faced with different phrasings of the same prompt. In this paper, we propose \emph{Flip-Flop Consistency} (F$^2$C), an unsupervised training method that improves robustness to such perturbations. F$^2$C is composed of two key components. The first, \emph{Consensus Cross-Entropy} (CCE), uses a majority vote across prompt variations to create a hard pseudo-label. The second is a representation alignment loss that pulls lower-confidence and non-majority predictors toward the consensus established by high-confidence, majority-voting variations.
We evaluate our method on 11 datasets spanning four NLP tasks, with 4--15 prompt variations per dataset. On average, F$^2$C raises observed agreement by 11.62\%, improves mean $F_1$ by 8.94\%, and reduces performance variance across formats by 3.29\%. In out-of-domain evaluations, F$^2$C generalizes effectively, increasing $\overline{F_1}$ and agreement while decreasing variance across most source-target pairs. Finally, when trained on only a subset of prompt perturbations and evaluated on held-out formats, F$^2$C consistently improves both performance and agreement while reducing variance.
These findings highlight F$^2$C as an effective unsupervised method for enhancing LLM consistency, performance, and generalization under prompt perturbations.\footnote{Code is available at our \href{https://github.com/ParsaHejabi/Flip-Flop-Consistency-Unsupervised-Training-for-Robustness-to-Prompt-Perturbations-in-LLMs}{GitHub repository}.}
\end{abstract}

\section{Introduction}
\begin{figure}[t]
  \centering
  \includegraphics[width=\linewidth]{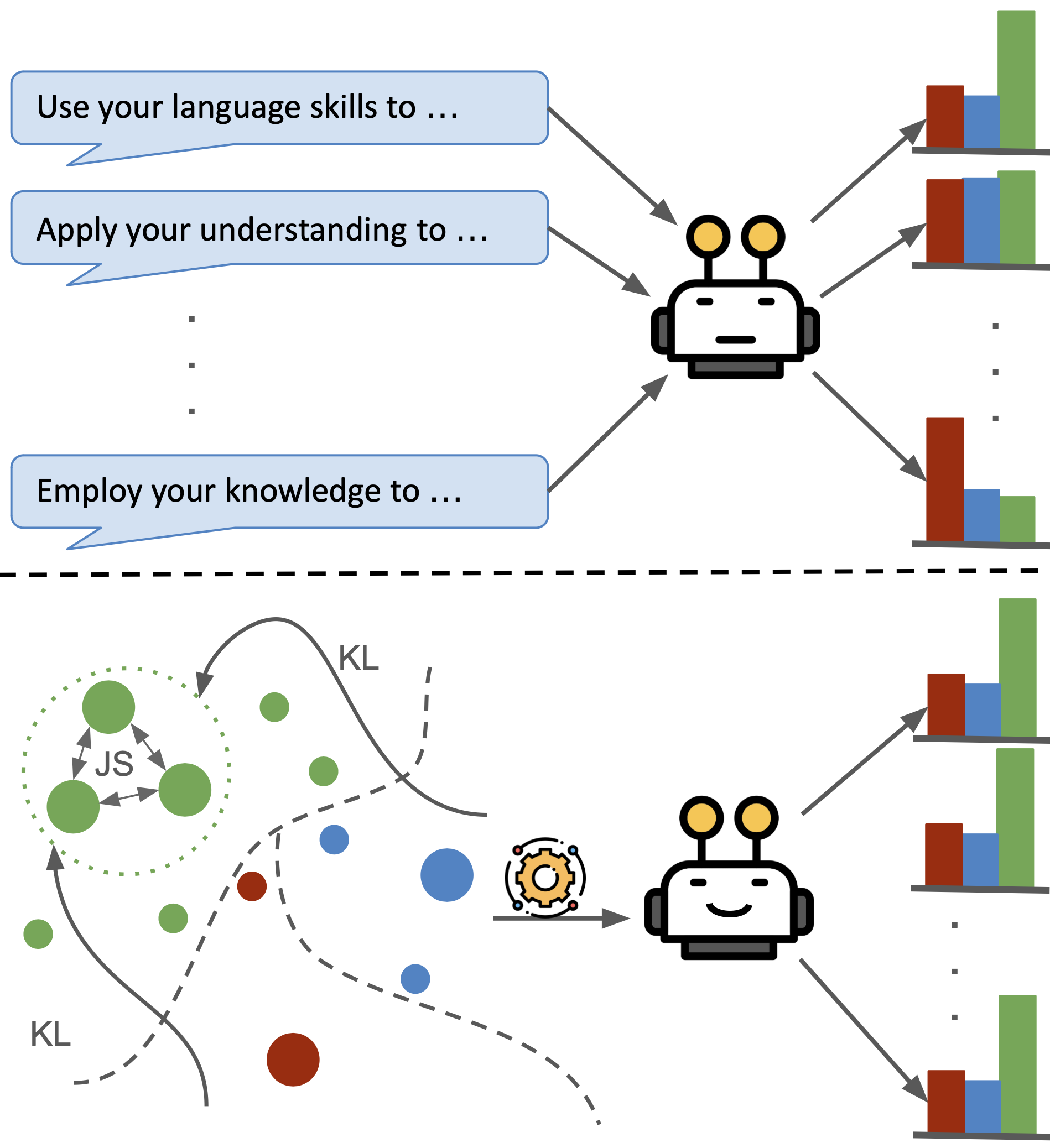}
  \caption{Our method aligns representations of input variations to promote consistency. To this end, we minimize the JS divergence among variations within the high-confidence consensus group, and the KL divergence between all other variations and that group.
  }
  \label{fig:main}
\end{figure}

Large Language Models (LLMs) are increasingly deployed across diverse domains, including high-stakes settings such as law and medicine \citep{openai2025gpt5systemcard, Singhal2025-bi, guha2023legalbench}, which raises the bar for reliability and trustworthiness. A core requirement for a trustworthy model is \emph{semantic consistency}: when the phrasing of a question varies but its meaning remains the same, the model's answer should remain consistent. Recent studies show that LLM predictions can vary sharply under prompt perturbations such as formatting, casing, separators, paraphrasing, item ordering in few-shot settings, and other surface changes, often shifting reported accuracy by large margins \citep{sclar2024quantifying, qiang-etal-2024-prompt, sun2024evaluating, lu-etal-2022-fantastically, cummins2025threat}. Accordingly, several works advocate reporting performance ranges (or variance) across prompt variants rather than a single point estimate \citep{mizrahi-etal-2024-state, polo2024efficient, alzahrani-etal-2024-benchmarks}.

While numerous studies evaluate the consistency in existing models and propose new metrics to quantify it \citep{chatterjee-etal-2024-posix, NEURIPS2024_7fa5a377, nalbandyan-etal-2025-score, datta-etal-2023-measuring}, fewer works aim to improve consistency within the models themselves. One line of work addresses this issue using prompt engineering techniques to search for the highest-performing prompt \citep{fu-etal-2024-learning, sclar2024quantifying, ngweta2025llmsrobustnesschangesprompt, NEURIPS2024_7fa5a377, raj2025semanticconsistencyassuringreliability, voronov-etal-2024-mind, salinas-morstatter-2024-butterfly}. However, while effective, these techniques add computational overhead for prompt optimization and do not resolve the internal inconsistency of the models. Other approaches try to solve this issue via supervised fine-tuning (SFT) \citep{qiang-etal-2024-prompt, yan-etal-2024-contrastive, sun2024evaluating, fu2025questiondifferentwordslatent}, although these methods are limited by the availability of labeled data. Lastly, inference-time intervention approaches try to address this issue through model editing \citep{yang-etal-2024-enhancing} and activation steering \citep{yang2025lfsteeringlatentfeatureactivation}. According to \citet{yang-etal-2024-enhancing}, despite being transparent, these methods fall behind SFT methods for improving performance.
In the unsupervised setting, \citet{zhou-etal-2022-prompt} propose ``swarm distillation,'' a pairwise consistency loss that aligns the representations of input variations. They use a diverse set of variations provided by the Public Pool of Prompts \citep[P3;][]{bach2022promptsourceintegrateddevelopmentenvironment} for input-output pairs of 11 datasets. However, \citet{NEURIPS2024_7fa5a377} show that while this method improves consistency, it decreases overall performance.

We target the gap of improving consistency in the absence of supervision while preserving task performance, and propose \emph{Flip-Flop Consistency} (F$^2$C), an unsupervised training algorithm that improves robustness to prompt perturbations by aligning their representations without sacrificing the performance. Prior work \citep{chen2025trulyneedsamplesmultillm, chen2024universal} has demonstrated that when a model consistently outputs the same label across variations, that label is more likely to be correct, whereas incorrect labels tend to be scattered, reflecting low confidence. Building on this insight, F$^2$C takes the majority answer across variations as a \emph{pseudo-label} for each data point. It then combines two components: (1) a cross-entropy loss that treats the pseudo-label as a hard label for all variations and (2) a divergence loss that aligns the distributions of less-confident and non-majority variations with those that confidently predict the majority. Together, these terms both increase the pseudo-label's probability and enforce consistency across variations.

We evaluate F$^2$C on 11 datasets through three studies: (1) a comprehensive analysis of its performance against the base model and swarm distillation; (2) testing out-of-domain (OOD) generalization, where a model trained on one dataset is evaluated on the others; and (3) examining generalization to unseen prompt perturbations, where the model is trained on the first $K$ prompt formats and evaluated on held-out formats. In \Cref{subsec:study1}, we demonstrate that F$^2$C significantly raises observed agreement ($P_o$) on average by 11.62\%, whereas swarm distillation slightly decreases it ($-$0.38\%), showing a +12.00\% agreement margin over swarm. As a beneficial byproduct, F$^2$C also improves $\overline{F_1}$ on 9 out of 11 datasets with an average gain of +8.94\% (vs.\ CCE: +8.36\%, swarm: +1.40\%) and reduces across-format $\sigma_{F_1}$ by 3.29\% on average (vs.\ CCE: 3.05\%, swarm: 0.47\%). For OOD generalization (\Cref{subsec:study2}), F$^2$C successfully generalizes to the OOD data, yielding higher $\overline{F_1}$ on 74/80 train$\rightarrow$test pairs, increases $P_o$ on 64/80, and lowers $\sigma_{F_1}$ on 66/80 compared to the base model. Finally, under limited format diversity (\Cref{subsec:study3}), we demonstrate that increasing the number of training formats in F$^2$C consistently lifts $\overline{F_1}$ and $P_o$ while shrinking $\sigma_{F_1}$, demonstrating robustness to unseen formats, despite only being trained on 5 or 10 prompt variations.

\section{Related Work}

\subsection{Evaluating Consistency in LLMs}
Despite strong zero-shot performance across many tasks \citep{NEURIPS2020_1457c0d6}, LLMs can be inconsistent-even contradictory-when responding to prompts that are semantically equivalent but phrased differently. Therefore, recent works advocate reporting performance as a \emph{range} across prompt variants, rather than a single score that may reflect only a best-case \citep{mizrahi-etal-2024-state, polo2024efficient, alzahrani-etal-2024-benchmarks, wang-etal-2024-assessing}. Empirical studies show large accuracy variations from simple format changes such as paraphrasing, casing, separators, spacing, and option ordering \citep{sclar2024quantifying, NEURIPS2024_7fa5a377, qiang-etal-2024-prompt, sun2024evaluating, lu-etal-2022-fantastically, cummins2025threat, alzahrani-etal-2024-benchmarks, wang-etal-2024-assessing}. Beyond raw accuracy spread, several frameworks and metrics target consistency more directly. \citet{nalbandyan-etal-2025-score} propose various non-adversarial perturbations, such as paraphrasing, option reordering, and temperature sampling with multiple independent samples, to yield more realistic estimates. In contrast, \citet{chatterjee-etal-2024-posix} argue that accuracy variance across templates overlooks response distribution and therefore cannot distinguish a model that is consistently wrong from one that produces different wrong answers depending on the template. The authors then introduce a sensitivity index, \texttt{POSIX}, capturing response overlap, entropy, semantic coherence, and confidence variation. For each meaning-preserving input format and its model-generated answer, it averages the difference of the probabilities of generating the same response across all prompt variants.

\subsection{Improving Consistency in LLMs}
Work on improving consistency spans several directions. A line of work addresses inconsistency via prompt engineering without changing model weights. \citet{fu-etal-2024-learning} train a small seq2seq ``paraphrase generator'' to rewrite queries into expressions the target LLM prefers.
Their method improves accuracy across QA, commonsense, and math tasks. \citet{ngweta2025llmsrobustnesschangesprompt} propose \emph{Mixture of Formats} (MOF), in which each few-shot example in the prompt uses a distinct format. \citet{raj2025semanticconsistencyassuringreliability} introduce \emph{Ask-to-Choose} (A2C), which samples multiple candidate answers and then prompts an LLM to select the best answer from those candidates. These approaches are effective at the prompt level but do not resolve the model’s internal inconsistency while incurring inference-time overhead to obtain a strong prompt.

Supervised training has also been used to improve consistency. \citet{yan-etal-2024-contrastive} take a contrastive-learning approach and make hidden states for paraphrased instructions with the same input-output pair closer and push apart hard negatives (same instruction, different input-output). They use paraphrasing to create perturbations for the training data; however, more diverse perturbations, e.g., typos, word substitutions, appending random sequences at the end of instructions, and multilingual paraphrases, are used for evaluation data. Their method improves robustness to unseen perturbed instructions, with an average accuracy gain of 2.5\% over continual instruction tuning. \citet{zhao-etal-2024-improving} introduce a two-stage alignment framework with two metrics, \emph{Consistency Rate} (pairwise agreement across paraphrases with an LLM-as-judge) and \emph{Maximum Consistency Rate} (the fraction of responses in the largest mutually consistent group). Stage~1 performs SFT on paraphrased instructions that share the same input–output, and Stage~2 generates multiple responses per input, scores them on format validity and correctness (using the gold label), forms preference pairs, and optimizes a DPO-style \citep{rafailov2023direct} ranking loss. A Vicuna-13B \citep{vicuna2023} model trained with this pipeline surpasses GPT-4 on CR. Similarly, \citet{qiang-etal-2024-prompt} propose Prompt Perturbation Consistency Learning (PPCL): during fine-tuning, they feed both a clean utterance and its perturbed version (oronyms, synonyms, or paraphrases) and optimize the cross-entropy on each. They also add a Jensen-Shannon (JS) divergence term between their token-level output distributions, which recovers much of the performance lost under prompt noise. \citet{sun2024evaluating} add small trainable soft-prompt embeddings and optimize them to make representations of semantically equivalent instructions more similar, which consistently improves zero-shot robustness to new phrasing. Finally, \citet{fu2025questiondifferentwordslatent} propose Latent Adversarial Paraphrasing (LAP): a bi-level scheme where an inner loop learns a constrained latent perturbation that acts as a continuous paraphrase while preserving semantics, and an outer loop fine-tunes the model on these perturbed inputs, improving worst-case win-rate by about 0.5–4\% without adding inference-time latency.

Model editing \citep{NEURIPS2022_6f1d43d5} and activation steering \citep{turner2024steeringlanguagemodelsactivation} have also been adapted to this problem. \citet{yang-etal-2024-enhancing} use a \emph{locate-then-edit} pipeline. They build paraphrase pairs, label each pair by whether the model's predictions agree (consistency), concatenate hidden states from both prompts, and train linear classifiers on last-token activations from each attention/MLP layer to predict the consistency label. They select the top-$K$ components with the highest classifier accuracy as key components for semantic consistency. For each component, they compute the difference between the mean hidden output of the consistent pairs and the mean over all pairs, and add this as a bias to that component's hidden state. This increases accuracy and reduces the across-variant standard deviation on NLU tasks, and increases mean pairwise cosine similarity across variants for NLG tasks. \citet{yang2025lfsteeringlatentfeatureactivation} use the same idea to identify the most influential transformer layer for consistency, then train a Top-K Sparse Autoencoder \citep[SAE,][]{gao2025scaling} to decompose its representation into a higher dimension. Using contrastive prompt pairs (correct vs.\ incorrect outputs), they select key SAE features with average activation differences exceeding a threshold and, at inference, add the learned feature offsets when the corresponding features activate, steering the model toward consistency. \citet{yang-etal-2024-enhancing} reported that, although these inference-time intervention methods are transparent, they generally fall behind SFT in performance.

Outside prompt-consistency settings, \citet{Shen_2024} improve robustness to textual adversarial attacks with a model-level dynamic-attention mechanism tailored to transformer architectures. Their method modifies attention through attention rectification and dynamic modeling, showing that model-level interventions can improve robustness even without relying solely on prompt search or standard supervised fine-tuning.

\citet{zhou-etal-2022-prompt} propose an unsupervised \emph{swarm distillation} loss. For each instance, they sample a pair of prompt formats with the same semantical meaning and apply pairwise distillation \citep{hinton2015distillingknowledgeneuralnetwork} so that one prompt's output distribution teaches the other (each prompt format can be both teacher and student). Using Fleiss' \emph{kappa} \citep{Fleiss1971-lv} to measure agreement across prompt variations, they report a relative 14.6\% increase over the T0-3B baseline on 8 out of 11 NLP datasets.

\citet{huang-etal-2023-large} leverage Self-Consistency decoding \citep{wang2023selfconsistency}, and sample multiple chain-of-thought solutions for each unlabeled question and take the majority answer as the ``high-confidence.'' They then retain all reasoning paths that yield the majority answer, convert each path into four mixed input-output formats, and then perform supervised fine-tuning on the resulting set and gain up to 7.7\% on GSM8K \citep{cobbe2021trainingverifierssolvemath}.

\section{Method}
To improve prompt perturbation robustness, we drew inspiration from two prior works. First, majority voting across prompt variations \citep{salinas-morstatter-2024-butterfly} builds on Self-Consistency \citep{wang2023selfconsistency} and achieves the highest overall accuracy across 11 classification tasks. Second, swarm distillation \citep{zhou-etal-2022-prompt} encourages consistency across prompt variations by minimizing KL divergence between all pairs. This is implemented via sequence-level distillation \citep{kim-rush-2016-sequence}, where each prompt simultaneously acts as both teacher and student. 

Building on these two ideas, we focus on an unsupervised setting where no gold labels are available during training and pseudo-labels must instead be inferred from the model's own responses across variations. Because the most frequent label produced across perturbations tends to represent the model's most confident and often correct prediction \citep{chen2024universal, chen2025trulyneedsamplesmultillm}, it is natural to treat the majority answer as a training signal.
However, plain cross-entropy alone does not guarantee consistency across semantically equivalent formats. It only increases the probability of the pseudo-labeled answer within each format without aligning distributions between formats.

Swarm distillation addresses this issue by enforcing agreement, pulling all variations' distributions toward their average (uniform mixture; see \autoref{thm:mixkl}) regardless of each variation's prediction. Yet, \citet{NEURIPS2024_7fa5a377} show that this averaging can harm overall performance, likely because the model overfits to noisy or lower quality mixtures. 

To address both limitations, we introduce ``Flip-Flop Consistency'' by combining two complementary components: (1) supervising with majority-vote pseudo-labels using Consensus Cross-Entropy (CCE; \Cref{subsec:cce}) and (2) aligning distributions to a stronger target, defined as the average distribution computed only from variations that confidently select the majority label (\Cref{subsec:ff}).

\subsection{Problem Formulation}

Let $\mathcal{T}$ be a classification task with $L$ labels $\mathcal{Y}=\{\,\ell_1,\dots,\ell_L\,\}$. The dataset consists of $N$ instances $\{(x_i,y_i)\}_{i=1}^N$, where $y_i\in\mathcal{Y}$ denotes the gold label. In our unsupervised setting, the gold labels are not used for training. We assume a set of $V$ prompt templates that preserve semantic meaning, $\mathcal{R}=\{\,r_1,\dots,r_V\,\}$. Each template $r_v$ renders inputs and label options:
\begin{equation}
x_i^{(v)} = r_v(x_i),
\label{eq:render-input}
\end{equation}
\begin{equation}
y_c^{(v)} = r_v(\ell_c)\quad\text{for } c=1,\dots,L.
\label{eq:render-label}
\end{equation}

\paragraph{Per-variation scoring.}
Using the model's length-normalized token-level log-likelihood for choosing label $\ell_c$ under template $v$, denoted $\mathrm{LL}_i[v,c]$ (see Appendix \ref{sec:appendix-implementation-details} for the exact computation), we define the per-variation label distribution as $\pi_{i,v,c}=\mathrm{softmax}(\mathrm{LL}_i[v,c])$ and the per-variation prediction
\begin{equation}
\hat{y}_{i,v} \;=\; \arg\max_{c\in\{1,\dots,L\}} \pi_{i,v,c}.
\label{eq:per-var-pred}
\end{equation}

\subsection{Consensus Cross-Entropy}
\label{subsec:cce}
We construct a pseudo-label via majority vote across variations and then fit the model to that label.

\paragraph{Consensus label.}
Define vote counts
\begin{equation}
n_{i,c}\;=\;\sum_{v=1}^{V}\mathbf{1}\!\left[\hat{y}_{i,v}=c\right],
\label{eq:cce-counts}
\end{equation}
and set the ``consensus'' (strict majority) label
\begin{equation}
c_i^\star\;=\;\arg\max_{c} n_{i,c}
\quad\text{with}\quad n_{i,c_i^\star}>\tfrac{V}{2};
\label{eq:cce-consensus}
\end{equation}
otherwise, no consensus is formed for instance $i$.

\paragraph{Loss.}
When a consensus exists for example $i$ with label $c_i^\star$, let $\ell_{i,v}$ denote the negative log-likelihood of the consensus answer $y_{c_i^\star}^{(v)}$ under variation $v$ given $x_i^{(v)}$ (scoring only the answer tokens). The instance-level CCE is
\begin{equation}
\mathcal{L}_{\mathrm{CCE}}(i)
\;=\;
\mathbf{1}\!\left[n_{i,c_i^\star}>\tfrac{V}{2}\right]\;
\lambda_{\mathrm{CCE}}\;\frac{1}{V}\sum_{v=1}^{V}\ell_{i,v},
\label{eq:cce-instance}
\end{equation}
and the training objective averages over examples:
\begin{equation}
\mathcal{L}_{\mathrm{CCE}}
\;=\;\frac{1}{N}\sum_{i=1}^{N}\mathcal{L}_{\mathrm{CCE}}(i).
\label{eq:cce-avg}
\end{equation}

Only examples with a strict majority contribute to the loss; if no consensus exists, $\mathcal{L}_{\mathrm{CCE}}(i)=0$. The coefficient $\lambda_{\mathrm{CCE}}$ controls this term's strength.

\subsection{Flip-Flop Consistency}
\label{subsec:ff}
We combine CCE with a representation alignment objective. Among the variations that vote for the consensus label $c_i^\star$, we identify confident prompts (the consensus-confident, or \emph{CC}, set) and align the remaining prompts (the non-confident or non-consensus, \emph{NC}, set) toward the CC set, while also encouraging agreement within the CC set. If no strict consensus exists, the example is skipped (no loss). We describe the details on forming the CC and NC sets in \autoref{alg:ff-controls}.

Given the strict majority consensus $c_i^\star$ and its consensus set $G=\{v:\hat{y}_{i,v}=c_i^\star\}$, the algorithm checks if a majority exists (line \ref{alg:ff-controls:line1}-\ref{alg:ff-controls:line2}), computes per-variation confidence margins $m_v$ by calculating the difference of log-likelihood between the consensus label and the most probable non-consensus label. Then, it takes the median $m_{\text{med}}$ as a representative for the consensus set (lines \ref{alg:ff-controls:line3}-\ref{alg:ff-controls:line5}). When $|G|{=}V$ and $m_{\text{med}}\!\ge\!\tau_{\text{unanimous}}$, all variations predict the same label confidently, so the algorithm returns $T_i{=}|G|$ (CC set) and $S_i{=}\emptyset$ (NC set) (line \ref{alg:ff-controls:line7}). If variations are not confident in producing the majority label or they produce different labels, the algorithm forms a CC/NC split to pull the NC set's representations toward the mean distribution of the CC set. This is done by picking the top-$k$ majority voter variations as the CC set and assigning the rest to the NC set (lines \ref{alg:ff-controls:line10}, \ref{alg:ff-controls:line13}-\ref{alg:ff-controls:line14}). Lastly, $w_{\text{flip}}$ weight is calculated to control the intensity of alignment between the CC and NC sets by applying a sigmoid function to the difference of average log-likelihoods for producing the consensus label between the CC and NC sets ($\Delta$). Finally, this weight is capped between $f_{\min}$ and $f_{\max}$ hyperparameters (lines \ref{alg:ff-controls:line15}-\ref{alg:ff-controls:line18}). Degenerate branches (fewer than two variations in the CC set) return empty sets and zero weight.

All loss components operate \emph{only} on the consensus answer tokens. For each variation $v$, let $\log\mathbf{q}_{i,v}^\star$ denote the model’s token-level log-softmax over the full vocabulary when outputting the consensus answer $y_{c_i^\star}^{(v)}$ under $x_i^{(v)}$, aggregated over answer positions. We use $\mathbf{q}_{i,v}^\star=\exp(\log\mathbf{q}_{i,v}^\star)$ inside divergence losses. For the CC set $T_i$, define the CC mixture $\bar{\mathbf{q}}_{i}^{T\star}$ as the (probability-space) average of $\{\mathbf{q}_{i,t}^\star\}_{t\in T_i}$.

Let $T_i$ (CC set) and $S_i$ (NC set) be the sets returned by Algorithm~\ref{alg:ff-controls}. There are three possible cases for each instance $i$:

\textbf{Case 1: No strict majority} ($|G|\le V/2$, line \ref{alg:ff-controls:line2}).
No pseudo-label is trusted; we skip the example and apply \emph{no loss}.

\textbf{Case 2: Unanimous \& confident} ($|G|{=}V$ and $m_{\text{med}}\!\ge\!\tau_{\text{unanimous}}$, line \ref{alg:ff-controls:line7}).
Here $T_i=G$, $S_i=\emptyset$, and $w_{\text{flip}}=0$. All variations are confident in outputting the majority answer, and in order to make them even more consistent, we apply a JSD loss with $\beta_{\text{jsd}}$ hyperparameter to make them even closer to their average point.
\begin{equation}
\mathcal{L}_{\text{jsd}}(i) = \beta_{\text{jsd}}\ \mathrm{JSD}\!\bigl(\{\mathbf{q}_{i,t}^\star\}_{t\in T_i}\bigr).
\label{eq:ff-jsd}
\end{equation}

\textbf{Case 3: Consensus with split} ($|T_i|\!\ge\!2$, lines \ref{alg:ff-controls:line5}-19).
In this case, we want to only pick the top $K$ most confident majority voters as the CC set and align the NC set toward their average distribution:
\begin{equation}
\mathcal{L}_{\text{flip}}(i)
= w_{\text{flip}}\ \frac{1}{|S_i|}\sum_{s\in S_i}\mathrm{KL}\!\bigl(\mathbf{q}_{i,s}^\star\,\|\,\bar{\mathbf{q}}_{i}^{T\star}\bigr),
\label{eq:ff-kl}
\end{equation}
and also encourage agreement within the CC set using the same $\mathcal{L}_{\text{jsd}}(i)$ as in Case~2 (\Cref{eq:ff-jsd}).
We control the strength $w_{\text{flip}}$ of $\mathcal{L}_{\text{flip}}(i)$ using the hyperparameters $f_{\min}, f_{\max}$ and temperature $t$.

\paragraph{Total loss and hyperparameters.}
To summarize, the total loss per example is:
\begin{equation}
\begin{aligned}
\mathcal{L}_{\text{FF}}(i) &= \\[-0.35ex]
&\begin{cases}
0, & \text{Case 1},\\
\mathcal{L}_{\mathrm{CCE}}(i)+\mathcal{L}_{\text{jsd}}(i), & \text{Case 2},\\
\mathcal{L}_{\mathrm{CCE}}(i)+\mathcal{L}_{\text{jsd}}(i)+\mathcal{L}_{\text{flip}}(i), & \text{Case 3}.\\
\end{cases}
\end{aligned}
\label{eq:ff-total}
\end{equation}

With hyperparameters
$\lambda_{\mathrm{CCE}}$ (CCE weight),
$\tau_{\text{unanimous}}$ (confidence threshold),
$k_{\max}$ (max CC size),
$f_{\min},f_{\max}$ and $t$ (flip loss caps, temperature),
$\beta_{\text{jsd}}$ (agreement weight within the CC set).

\SetAlFnt{\small}
\begin{algorithm}[!tb]
\caption{F$^2$C for instance $i$}
\label{alg:ff-controls}

\KwData{%
  $\mathrm{LL}_i \in \mathbb{R}^{V\times L}$,\\
  Consensus label $c_i^\star$,\\
  Consensus set $G=\{\,v:\hat{y}_{i,v}=c_i^\star\,\}$%
} 
\KwIn{%
  Unanimous margin $\tau_{\text{unanimous}}$,\\
  CC set size cap $k_{\max}\!\ge\!2$,\\
  weight bounds $f_{\min}\!\le\!f_{\max}$,\\
  temperature $t\!>\!0$%
}
\KwResult{%
  CC set $T_i$,\\
  NC set $S_i$,\\
  flip weight $w_{\text{flip}}$%
}
\BlankLine

\If{$|G|\le V/2$}{\label{alg:ff-controls:line1}
  \Return{$(\emptyset,\emptyset,0)$} \tcp*{no strict majority}
}\label{alg:ff-controls:line2}

\ForEach{$v\in G$}{\label{alg:ff-controls:line3}
  $m_v \gets \mathrm{LL}_i[v,c_i^\star] - \max_{c\neq c_i^\star}\mathrm{LL}_i[v,c]$\;
}\label{alg:ff-controls:line4}
$m_{\text{med}} \gets \mathrm{median}\{m_v : v\in G\}$\;\label{alg:ff-controls:line5}

\If{$|G|=V$ \textbf{and} $m_{\text{med}}\ge \tau_{\text{unanimous}}$}{
  \Return{$(G,\emptyset,0)$} \tcp*{unanimous \& confident}
}\label{alg:ff-controls:line7}

\If{$|G|<2$}{
  \Return{$(\emptyset,\emptyset,0)$}
}

$k \gets \min\bigl(k_{\max},\,V-1\bigr)$ \tcp*{leave at least one variation in NC set}\label{alg:ff-controls:line10}
\If{$k<2$}{
  \Return{$(\emptyset,\emptyset,0)$} \tcp*{need $\ge2$ variations in CC}
}

$T_i \gets \text{top-}k$ members of $G$ by $m_v$ (descending)\;\label{alg:ff-controls:line13}
$S_i \gets \{1,\dots,V\}\setminus T_i$\;\label{alg:ff-controls:line14}

$\bar{\ell}_T \gets \frac{1}{|T_i|}\sum_{t\in T_i}\mathrm{LL}_i[t,c_i^\star]$\;\label{alg:ff-controls:line15}
$\bar{\ell}_S \gets \frac{1}{|S_i|}\sum_{s\in S_i}\mathrm{LL}_i[s,c_i^\star]$\;
$\Delta \gets \bar{\ell}_T - \bar{\ell}_S$ \tcp*{gap on consensus label}
$w_{\text{flip}} \gets f_{\min} + (f_{\max}-f_{\min})\cdot \sigma(\Delta/t)$\;\label{alg:ff-controls:line18}

\Return{$(T_i,S_i,w_{\text{flip}})$}
\end{algorithm}

\subsection{Metrics}
Following \citet{zhao-etal-2024-improving} we use the raw observed agreement ($P_o$) to measure consistency across prompt variations. We omit Fleiss’ $\kappa$ used in \citep{zhou-etal-2022-prompt} due to the prevalence/bias paradox noted by \citet{Hoehler2000-qg}. Using the vote counts $n_{i,c}$ from Eq.~\ref{eq:cce-counts}, the per-item agreement is calculated as in Eq.~\ref{eq:raw_agreement}, and $P_o$ is the average of $P_i$ over items. Intuitively, $P_i$ represents the probability that two uniformly sampled prompt variations for the same input predict the same label.
\begin{equation}
P_i \;=\; \frac{1}{V(V-1)}\sum_{c=1}^{L} n_{i,c}\,\bigl(n_{i,c}-1\bigr),
\label{eq:raw_agreement}
\end{equation}
High agreement alone may result from a collapsed model that predicts a single label for all prompts. To ensure consistency does not come at the expense of task performance, we also report $\overline{F_1}$. Prior work has further used performance spread- best- vs. worst-case performance- \citep{sclar2024quantifying}, but this metric is highly sensitive to outliers. Instead, we report the standard deviation $\sigma_{F_1}$ across prompt variations to capture performance stability.

\section{Experimental Setup}
\subsection{Datasets}
Following \citet{zhou-etal-2022-prompt}, we evaluate our proposed method on eleven classification datasets spanning four tasks, and use templates from the Public Pool of Prompts \citep[P3;][]{bach2022promptsourceintegrateddevelopmentenvironment} to create prompt variations. The NLP tasks include natural language inference (\citep[ANLI R1/R2/R3,][]{nie-etal-2020-adversarial}, \citep[CB,][]{De_Marneffe2019-wa}, \citep[RTE,][]{10.5555/3454287.3454581}), sentence completion (COPA \citep{roemmele_choice_2011}, HellaSwag \citep{zellers-etal-2019-hellaswag}, StoryCloze 2016 \citep{mostafazadeh-etal-2016-corpus}), coreference-style commonsense (WSC \citep{10.5555/3454287.3454581}, WinoGrande-XL \citep{ai2:winogrande}), and word sense disambiguation (WiC \citep{pilehvar-camacho-collados-2019-wic}).
Most official test splits of these datasets are unlabeled. Therefore, we evaluate on the official \emph{validation} split. We create a stratified hold-out set from the official training split for validation. Details and statistics of all datasets are reported in the Appendix~\ref{sec:dataset}.

\subsection{Implementation Details}
\label{sec:exp_details}
We compare our approach against two baselines, the unmodified base model and the base model fine-tuned with swarm distillation. For all experiments, we fine-tune \texttt{Qwen2.5-3B-Instruct}~\citep{qwen2025qwen25technicalreport} using LoRA \citep{hu2022lora}. The first experiment is also run on \texttt{Llama-3.2-3B-Instruct}~\citep{grattafiori2024llama3herdmodels}, where we compare the base model with its F$^2$C fine-tuned variant. Configuration details are provided in Appendix~\ref{paragraph:appendix-lora-config}.

\begin{table*}[t]
  \centering
  \setlength{\tabcolsep}{3.7pt}
  \resizebox{\linewidth}{!}{%
  \begin{tabular}{ll c c c c c c c c c c c}
    \toprule
     &  & ANLI R1 & ANLI R2 & ANLI R3 & CB & RTE & COPA & HellaSwag & StoryCloze & WSC & WinoGrande & WiC \\
    \midrule
    \multirow{3}{*}{Base} & $\overline{F_1}$ & 39.23 & 32.31 & 31.02 & \textcolor{Orange}{31.70} & \textcolor{Orange}{77.90} & 81.73 & 35.43 & 85.35 & 42.39 & \textcolor{Orange}{54.43} & 11.45 \\
     & $\sigma_{F_1}$ & 13.97 & 9.76 & 9.35 & 19.46 & \textcolor{Orange}{4.12} & 10.65 & 2.20 & \textcolor{Orange}{7.78} & \textbf{\textcolor{RoyalBlue}{11.84}} & \textcolor{Orange}{3.69} & \textcolor{Orange}{17.53} \\
     & $P_o$ & 52.29 & 52.80 & 52.84 & 46.73 & \textcolor{Orange}{85.08} & 81.54 & \textcolor{Orange}{50.89} & 85.24 & 74.29 & \textcolor{Orange}{68.87} & \textcolor{Orange}{85.00} \\
    \midrule
    \multirow{3}{*}{Swarm} & $\overline{F_1}$ & \textcolor{Orange}{38.77} & \textcolor{Orange}{32.06} & \textcolor{Orange}{30.95} & \textbf{\textcolor{RoyalBlue}{34.32}} & 78.34 & \textcolor{Orange}{81.35} & \textcolor{Orange}{34.92} & \textcolor{Orange}{84.69} & \textbf{\textcolor{RoyalBlue}{46.06}} & 62.30 & \textbf{\textcolor{RoyalBlue}{14.57}} \\
     & $\sigma_{F_1}$ & \textcolor{Orange}{14.27} & \textcolor{Orange}{9.92} & \textcolor{Orange}{9.36} & \textcolor{Orange}{21.01} & 2.31 & \textcolor{Orange}{10.95} & 2.26 & 7.64 & \textcolor{Orange}{14.40} & 2.66 & \textbf{\textcolor{RoyalBlue}{10.44}} \\
     & $P_o$ & \textcolor{Orange}{50.91} & \textcolor{Orange}{50.63} & \textcolor{Orange}{50.49} & \textcolor{Orange}{40.78} & 86.94 & \textcolor{Orange}{80.93} & 50.97 & \textcolor{Orange}{85.23} & \textcolor{Orange}{68.58} & 75.22 & \textbf{\textcolor{RoyalBlue}{90.73}} \\
    \midrule
    \multirow{3}{*}{CCE} & $\overline{F_1}$ & \textbf{\textcolor{RoyalBlue}{58.68}} & 39.42 & 34.08 & 32.14 & 81.26 & 90.44 & 67.64 & 96.22 & \textcolor{Orange}{41.43} & 63.88 & \textcolor{Orange}{9.72} \\
     & $\sigma_{F_1}$ & \textbf{\textcolor{RoyalBlue}{2.40}} & \textbf{\textcolor{RoyalBlue}{5.97}} & 7.07 & 19.27 & 1.54 & 5.52 & \textcolor{Orange}{2.80} & \textbf{\textcolor{RoyalBlue}{0.38}} & 12.16 & 2.99 & 16.73 \\
     & $P_o$ & \textbf{\textcolor{RoyalBlue}{78.96}} & \textbf{\textcolor{RoyalBlue}{70.58}} & \textbf{\textcolor{RoyalBlue}{72.35}} & \textbf{\textcolor{RoyalBlue}{50.39}} & 91.71 & 90.32 & 72.42 & 97.05 & \textbf{\textcolor{RoyalBlue}{75.91}} & 77.16 & 87.16 \\
    \midrule
    \multirow{3}{*}{F$^2$C} & $\overline{F_1}$ & 58.43 & \textbf{\textcolor{RoyalBlue}{39.58}} & \textbf{\textcolor{RoyalBlue}{35.77}} & 31.81 & \textbf{\textcolor{RoyalBlue}{81.55}} & \textbf{\textcolor{RoyalBlue}{90.84}} & \textbf{\textcolor{RoyalBlue}{70.28}} & \textbf{\textcolor{RoyalBlue}{96.32}} & 41.68 & \textbf{\textcolor{RoyalBlue}{64.98}} & 9.99 \\
     & $\sigma_{F_1}$ & 2.61 & 6.17 & \textbf{\textcolor{RoyalBlue}{6.24}} & \textbf{\textcolor{RoyalBlue}{19.19}} & \textbf{\textcolor{RoyalBlue}{0.82}} & \textbf{\textcolor{RoyalBlue}{5.39}} & \textbf{\textcolor{RoyalBlue}{2.12}} & 0.46 & 12.12 & \textbf{\textcolor{RoyalBlue}{2.43}} & 16.66 \\
     & $P_o$ & 78.85 & 67.64 & 68.68 & 50.26 & \textbf{\textcolor{RoyalBlue}{93.40}} & \textbf{\textcolor{RoyalBlue}{90.71}} & \textbf{\textcolor{RoyalBlue}{76.97}} & \textbf{\textcolor{RoyalBlue}{97.32}} & 75.23 & \textbf{\textcolor{RoyalBlue}{77.30}} & 87.01 \\
    \bottomrule
  \end{tabular}}
    \caption{Comparison across datasets for \texttt{Qwen2.5-3B-Instruct}: the base model and variants trained with Swarm (swarm distillation), CCE, and F$^2$C. Bold blue values mark the best metric per dataset column, orange values denote the worst.}
  \label{aggregate_all_datasets_pivoted}
\end{table*}

\section{Results}

We conduct three experiments to analyze whether our method: (1) improves robustness to prompt perturbations without reducing task performance, (2) maintains semantic consistency in out-of-domain settings, and (3) is not bound to the prompt formats used in training and generalizes to unseen formats.

Before evaluating our method, we assess the inherent consistency of the \texttt{Qwen2.5-3B-Instruct} base model (see Appendix \autoref{fig:app-f1-spread}). The mean interquartile range (Q3$-$Q1) of $F_1$ across prompt variations, averaged over datasets, is 13.53\%. Nearly half of the datasets (5/11) exceed a 15\% spread, indicating substantial inconsistency within the base model.

F$^2$C is motivated by prior evidence that majority pseudo-labels across prompt variations are more likely to be correct. Therefore, we evaluate the strict-majority labels produced by the base \texttt{Qwen2.5-3B-Instruct} model in our setting (see Appendix~\ref{paragraph:appendix-pseudo-label-quality}). A strict-majority label exists for 93.9\% of validation examples across their prompt formats, and the $F_1$ of those majority labels is 59.5 when per-dataset scores are averaged by the number of instances with strict-majority label. The signal is strongest on RTE, COPA, and StoryCloze, moderate on ANLI, HellaSwag, and WinoGrande, and weakest on CB, WSC, and especially WiC. These results help interpret the dataset-level gains and failure cases in the first experiment.

\subsection{Flip-Flop Consistency Against Baselines}
\label{subsec:study1}

To examine whether aligning representations across prompt formats further improves consistency when combined with the CCE loss, we train \texttt{Qwen2.5-3B-Instruct} using two loss functions, CCE, and F$^2$C, and compare them against the baselines introduced in \Cref{sec:exp_details}. Since training-induced gains can depend on the model family \citep{shao2026spuriousrewardsrethinkingtraining}, we also compare the base \texttt{Llama-3.2-3B-Instruct} with its F$^2$C fine-tuned variant. Implementation details are provided in \Cref{paragraph:appendix-model-selection}. For each method and dataset, \autoref{aggregate_all_datasets_pivoted} and \autoref{aggregate_all_datasets_pivoted_llama} report $\overline{F_1}$, $\sigma_{F_1}$, and $P_o$ for \texttt{Qwen2.5-3B-Instruct} and \texttt{Llama-3.2-3B-Instruct}, respectively.

On \texttt{Qwen2.5-3B-Instruct}, F$^2$C achieves the largest average improvements over the base model. It improves agreement by 11.62\%, increases the mean $F_1$ by 8.94\%, and reduces ${F_1}$ variance by 3.29\%. In comparison, CCE raises the average agreement by 11.68\%, with slightly smaller gains in mean $F_1$ and $\sigma_{F_1}$ reduction (8.36\% and 3.05\%, respectively). Both F$^2$C and CCE achieve higher $\overline{F_1}$ and $P_o$ and lower $\sigma_{F_1}$ than the baselines on most datasets, including ANLI~R1/R2/R3, RTE, COPA, HellaSwag, StoryCloze, and WinoGrande. In contrast, swarm distillation is the weakest baseline that reduces agreement on average by 0.38\% and produces only a small mean $F_1$ gain of 1.40\%. These results indicate that F$^2$C not only enhances consistency but also improves task performance.

\begin{table*}[t]
  \centering
  \setlength{\tabcolsep}{3.7pt}
  \resizebox{\linewidth}{!}{%
  \begin{tabular}{ll c c c c c c c c c c c}
    \toprule
     &  & ANLI R1 & ANLI R2 & ANLI R3 & CB & RTE & COPA & HellaSwag & StoryCloze & WSC & WinoGrande & WiC \\
    \midrule
    \multirow{3}{*}{Base} & $\overline{F_1}$ & \textcolor{Orange}{32.28} & \textcolor{Orange}{29.43} & \textcolor{Orange}{28.95} & \textcolor{Orange}{25.00} & \textbf{\textcolor{RoyalBlue}{69.41}} & \textcolor{Orange}{68.21} & \textcolor{Orange}{24.16} & \textcolor{Orange}{65.02} & \textcolor{Orange}{44.74} & \textcolor{Orange}{52.23} & \textbf{\textcolor{RoyalBlue}{35.36}} \\
     & $\sigma_{F_1}$ & \textcolor{Orange}{8.56} & \textcolor{Orange}{6.69} & \textcolor{Orange}{7.32} & \textcolor{Orange}{12.00} & \textbf{\textcolor{RoyalBlue}{3.48}} & \textcolor{Orange}{5.77} & \textcolor{Orange}{7.40} & \textcolor{Orange}{7.64} & \textcolor{Orange}{8.76} & \textcolor{Orange}{3.60} & \textcolor{Orange}{22.98} \\
     & $P_o$ & \textcolor{Orange}{48.08} & \textcolor{Orange}{46.78} & \textcolor{Orange}{45.09} & \textcolor{Orange}{40.03} & \textbf{\textcolor{RoyalBlue}{77.46}} & \textcolor{Orange}{77.64} & \textcolor{Orange}{52.28} & \textcolor{Orange}{80.72} & \textcolor{Orange}{52.26} & \textcolor{Orange}{80.90} & \textcolor{Orange}{59.48} \\
    \midrule
    \multirow{3}{*}{F$^2$C} & $\overline{F_1}$ & \textbf{\textcolor{RoyalBlue}{36.36}} & \textbf{\textcolor{RoyalBlue}{31.48}} & \textbf{\textcolor{RoyalBlue}{31.36}} & \textbf{\textcolor{RoyalBlue}{25.22}} & \textcolor{Orange}{68.80} & \textbf{\textcolor{RoyalBlue}{84.35}} & \textbf{\textcolor{RoyalBlue}{39.64}} & \textbf{\textcolor{RoyalBlue}{94.20}} & \textbf{\textcolor{RoyalBlue}{54.58}} & \textbf{\textcolor{RoyalBlue}{66.75}} & \textcolor{Orange}{29.80} \\
     & $\sigma_{F_1}$ & \textbf{\textcolor{RoyalBlue}{3.87}} & \textbf{\textcolor{RoyalBlue}{4.83}} & \textbf{\textcolor{RoyalBlue}{4.66}} & \textbf{\textcolor{RoyalBlue}{9.63}} & \textcolor{Orange}{3.80} & \textbf{\textcolor{RoyalBlue}{5.59}} & \textbf{\textcolor{RoyalBlue}{3.40}} & \textbf{\textcolor{RoyalBlue}{0.56}} & \textbf{\textcolor{RoyalBlue}{1.96}} & \textbf{\textcolor{RoyalBlue}{0.15}} & \textbf{\textcolor{RoyalBlue}{22.97}} \\
     & $P_o$ & \textbf{\textcolor{RoyalBlue}{73.51}} & \textbf{\textcolor{RoyalBlue}{54.40}} & \textbf{\textcolor{RoyalBlue}{65.71}} & \textbf{\textcolor{RoyalBlue}{54.42}} & \textcolor{Orange}{75.82} & \textbf{\textcolor{RoyalBlue}{85.11}} & \textbf{\textcolor{RoyalBlue}{59.17}} & \textbf{\textcolor{RoyalBlue}{96.12}} & \textbf{\textcolor{RoyalBlue}{88.80}} & \textbf{\textcolor{RoyalBlue}{98.91}} & \textbf{\textcolor{RoyalBlue}{65.07}} \\
    \bottomrule
  \end{tabular}}
    \caption{Comparison across datasets for \texttt{Llama-3.2-3B-Instruct}: the base model and the variant trained with F$^2$C. Bold blue values mark the best metric per dataset column, orange values denote the worst.}
  \label{aggregate_all_datasets_pivoted_llama}
\end{table*}

The same qualitative pattern holds on \texttt{Llama-3.2-3B-Instruct}. Relative to the base model, F$^2$C improves $\overline{F_1}$ on 9/11 datasets, lowers $\sigma_{F_1}$ on 10/11, and increases $P_o$ on 10/11, with average changes of +7.98\%, -2.98\%, and +14.21\%, respectively. Only RTE and WiC do not see mean $F_1$ gains, while the remaining datasets follow the same overall trend as \texttt{Qwen2.5-3B-Instruct}. This indicates that the gains are not specific to the Qwen family.

On \texttt{Qwen2.5-3B-Instruct}, the consensus is unreliable in cases where the base model has weak performance and high $\sigma_{F_1}$ (CB, WSC, and WiC). This is consistent with the pseudo-label analysis above: these are also the datasets where strict-majority labels are least reliable. Therefore, pushing toward consensus does not improve performance; nevertheless, F$^2$C and CCE still raise $P_o$ in these datasets.

Overall, F$^2$C increases agreement, improves task performance, and reduces across-format variance, while swarm distillation can even harm agreement.

\subsection{Generalization in Out-of-Domain Settings}
\label{subsec:study2}

We assess out-of-domain (OOD) generalization by evaluating a model trained on source dataset with F$^2$C on all other target datasets. CB, WSC, and WiC are excluded as sources due to weak performance in \Cref{subsec:study1} but are retained as targets. For each source$\rightarrow$target pair, we compute the change in mean $\overline{F_1}$, $\sigma_{F_1}$, and agreement $P_o$ on the target dataset relative to the base model (see \autoref{tab:study2_merged_summary_compact}). Appendix Figs.~\ref{fig:hm:raw_agreement}, \ref{fig:hm:meanf1}, and \ref{fig:hm:stdf1} visualize the corresponding $\Delta\overline{F_1}$, $\Delta P_o$, and $\Delta\sigma_{F_1}$ heatmaps.

Overall, F$^2$C generalizes well across domains. Averaged over all 80 dataset pairs, observed agreement increases by 7.49\%, mean $F_1$ by 7.61\%, and $\sigma_{F_1}$ decreases by 2.94\%. Moreover, positive transfers substantially outnumber negatives across all three metrics (P/N columns).

Training on story/commonsense datasets such as COPA, and StoryCloze yields the strongest average improvements in $\overline{F_1}$, while WinoGrande produces the largest average gains in $P_o$ and the greatest reduction in $\sigma_{F_1}$. RTE and ANLI~R1 also generalize reliably across many targets. Harder targets such as WSC and WiC show smaller or mixed changes in agreement, though variance typically still declines (see Appendix heatmaps~\ref{fig:hm:raw_agreement}, \ref{fig:hm:meanf1}, and \ref{fig:hm:stdf1} for per-pair patterns).

\begin{table}[t]
  \centering
  \small

  \setlength{\tabcolsep}{3.5pt}
  \resizebox{\columnwidth}{!}{%
  \begin{tabular}{lcccccc}
    \toprule
     & \multicolumn{2}{c}{$\overline{F_1}$} & \multicolumn{2}{c}{$P_o$} & \multicolumn{2}{c}{$\sigma_{F_1}$} \\
    \cmidrule(lr){2-3} \cmidrule(lr){4-5} \cmidrule(lr){6-7}
    Train Dataset & $\Delta$ & P/N & $\Delta$ & P/N & $\Delta$ & P/N \\
    \midrule
    All (80 pairs) & 7.613 & 74/6 & 7.491 & 64/16 & 2.941 & 66/14 \\
    \midrule
    ANLI R1 & 9.428 & 10/0 & 8.727 & 7/3 & 3.653 & 8/2 \\
    ANLI R2 & 5.543 & 10/0 & 6.886 & 8/2 & 2.171 & 8/2 \\
    ANLI R3 & 4.210 & 7/3 & 6.801 & 9/1 & 2.225 & 9/1 \\
    RTE & 8.077 & 10/0 & 5.938 & 8/2 & 3.255 & 9/1 \\
    COPA & \textbf{11.836} & 10/0 & 7.350 & 8/2 & 3.077 & 8/2 \\
    HellaSwag & 2.810 & 7/3 & 8.074 & 8/2 & 2.410 & 7/3 \\
    StoryCloze & 11.176 & 10/0 & 7.267 & 7/3 & 2.896 & 8/2 \\
    WinoGrande & 7.823 & 10/0 & \textbf{8.882} & 9/1 & \textbf{3.839} & 9/1 \\
    \bottomrule
  \end{tabular}
  }
    \caption{Cross-dataset transfer performance under F$^2$C. Each row shows the mean signed $\Delta$ relative to the base model when the \emph{row} dataset is used for training. Columns report changes in $\overline{F_1}$, $P_o$, and $\sigma_{F_1}$, along with \emph{P/N} (number of datasets with positive or negative improvement out of 10). The top ``All (80 pairs)'' row aggregates over all source$\rightarrow$target pairs. Bold numbers indicate the best dataset for each metric.}
  \label{tab:study2_merged_summary_compact}
\end{table}

\subsection{Generalization to Unseen Variations}
\label{subsec:study3}
We test whether F$^2$C trained on a subset of prompt formats generalizes to \emph{unseen} formats. We use ANLI~R1/R2/R3 (15 formats each) and RTE (10 formats) due to their larger instance size and number of available variations. For RTE, we train with the first 5 formats and evaluate on the remaining 5. For each ANLI dataset, we hold out the last 5 formats for evaluation and train with the first 5, then with 10 (see \autoref{fig:study3_combined}).

Across four datasets, using more training formats improves both performance and agreement on the held-out formats, and reduces across-format variance. ANLI~R1 shows the largest steady gains as the number of training variations increases. ANLI~R2 improves moderately but monotonically. ANLI~R3 shows a small decrease with 10 variations but improves with 15 variations. RTE is strong even at 5 formats and still rises with more. Error bands ($\sigma_{F_1}$ over held-out formats) shrink as we add formats, indicating higher semantic consistency.

\subsection{Representation Alignment Analysis}
\label{subsec:representation_alignment}

To better understand how F$^2$C changes the model internally, we analyze the final-layer hidden state at the answer boundary, immediately before any answer token is generated. For each dataset instance, we extract this vector for every prompt variation, compute all pairwise cosine similarities among the prompt variations, and average them into one score for that instance. F$^2$C increases this within-instance similarity on 10/11 datasets, with an average gain of $0.040$ (see Appendix~\autoref{fig:representation_alignment}). The only decrease is on CB ($-0.021$), consistent with its weaker performance in \Cref{subsec:study1}. This shows that F$^2$C generally maps semantically equivalent prompt variants to more similar internal states, rather than only increasing output-level agreement.

Higher similarity across prompt variants alone could be caused by representation collapse. We therefore also examine whether class-level distinctions are preserved after fine-tuning with F$^2$C. For each instance, we average its prompt variation representations into one instance vector. We then group instances by their dataset-provided gold labels and compute the centroid of the instance vectors for each class. The within-class distance is the average cosine distance from each instance vector to its own class centroid; lower values mean that examples with the same gold label are closer together. The between-class distance is the average cosine distance between class centroids; higher values mean that different classes are more separated. This analysis is more interpretable for RTE and ANLI R1/R2/R3, where labels denote the same sentence-pair relations across examples, unlike multiple-choice tasks where label IDs often only identify answer positions. On RTE and ANLI R1/R2/R3, F$^2$C reduces within-class distance while maintaining or increasing between-class distance. For example, on RTE, between-class distance increases from $0.028$ to $0.032$, while within-class distance decreases from $0.022$ to $0.018$. Thus, on these NLI tasks, F$^2$C reduces prompt-induced variation without collapsing the class distinctions. Full results are provided in Appendix~\ref{app:representation_alignment}.

\begin{figure}[t]
  \centering
  \includegraphics[width=0.95\linewidth]{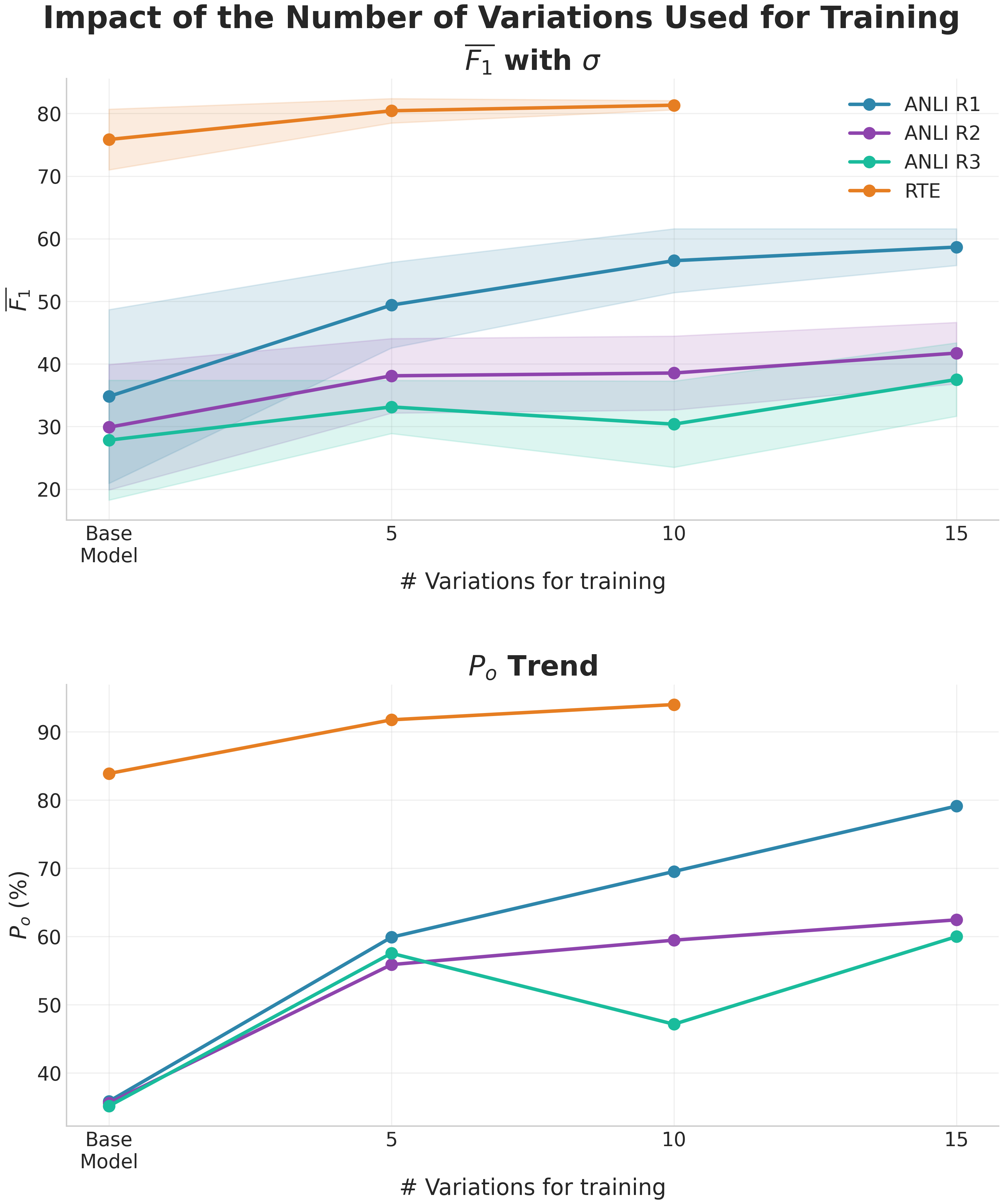}
  \caption{Top: $\overline{F_1}$ with shaded $\sigma_{F_1}$ on the \emph{held-out} prompt formats.
  Bottom: observed agreement $P_o$ on the same held-out sets.
  The x-axis is the number of training formats ($K$); ``Base Model'' is the untrained baseline.
  }
  \label{fig:study3_combined}
\end{figure}

\section{Conclusion}
LLMs often change their predictions when semantically equivalent prompts are phrased differently, undermining consistency and reliability. To address the problem of semantic consistency, we introduced F$^2$C, an unsupervised training method that uses majority voting across prompt variations to form hard pseudo-labels, then selectively aligns distributions toward confident majority voters while encouraging agreement among them. Our method relies on signals the model already treats as reliable, reinforcing them to ensure consistency while minimizing the influence of noisy or uncertain variations.

To validate our method, we conducted comprehensive experiments across multiple datasets and generalization scenarios. Across 11 datasets, F$^2$C consistently improves agreement, increases mean $F_1$, and reduces variance across prompt formats, outperforming both swarm distillation and the CCE-only variant. These gains persist in two generalization settings: cross-dataset transfer and generalization to unseen prompt formats, in which training on only a subset of formats still improves performance and agreement on held-out ones.

Our results suggest that much of the inconsistency from prompt phrasing can be mitigated by leveraging the model’s own internal consensus, without gold labels. Future work includes extending F$^2$C to open-ended generation, exploring adaptive selection of high-confidence variations beyond top-$K$, and combining our approach with lightweight supervision when labels are available.

\section*{Limitations}
While our results indicate consistent gains, several limitations should be acknowledged. First, although the first experiment is reported on two 3B instruction-tuned model families (\texttt{Qwen2.5-3B-Instruct} and \texttt{Llama-3.2-3B-Instruct}), the out-of-domain and unseen-variation generalization analyses are run only on \texttt{Qwen2.5-3B-Instruct}. Consequently, the scalability of F$^2$C to larger or smaller models, different pretraining corpora, or non-instruction-tuned bases is not established. Second, our evaluation is specifically focused on classification tasks with discrete labels. We do not study open-ended generation (e.g., long-form QA or chain-of-thought), for which our proposed method and evaluation metrics may require adaptation. Third, the perturbations we consider are non-adversarial template variants drawn from PromptSource. We do not test robustness to stronger or adversarial edits (character-, word-, or sentence-level changes), jailbreak-style attacks, multilingual rewrites, or heavy formatting noise. Beyond coverage, F$^2$C assumes multiple semantically equivalent templates per instance and uses them during training. In settings with scarce or low-quality templates, effectiveness and efficiency may degrade. Methodologically, we rely on the majority pseudo-labels and skip instances without a majority. On small or class-imbalanced datasets, the majority may be wrong, and the skip rule can bias learning toward ``easier'' examples, despite our confidence-aware variation selection for CC set. For model selection, we use validation $\overline{F_1}$ even though our objectives also target agreement ($P_o$) and dispersion ($\sigma_{F_1}$). Alternative criteria (e.g., multi-objective or worst-case) could yield different trade-offs, and we do not evaluate calibration or abstention. In terms of generalization, our target datasets are across related English NLP classification datasets, not across modalities, code, tool-use tasks, or languages. Finally, F$^2$C introduces several hyperparameters (e.g., CC set size cap, confidence thresholds, temperature) that we do not exhaustively tune.

\section*{Ethical Considerations}
We use only publicly available datasets and open-weight models, with no new human data collection. All datasets used in this work are established benchmarks obtained via the Hugging Face \texttt{datasets} library~\citep{lhoest2021datasetscommunitylibrarynatural} or their official repositories, each under its original license. These corpora are designed for evaluating language understanding and reasoning and contain de-identified, non-sensitive text drawn from newswire, Wikipedia, instructional materials, or crowdsourced fictional narratives. While some datasets may include named entities (e.g., public figures in news excerpts), to the best of our knowledge none contain contact details or other sensitive personal identifiers.

\bibliography{references}

\appendix

\section{Appendix}
\label{sec:appendix}
\subsection{Why swarm distillation moves all students towards their average?}
\begin{theorem}[Mixture-teacher decomposition]\label{thm:mixkl}
Let $Y$ be finite. For a fixed input with $K$ semantically equivalent formats, let
\(q_1,\dots,q_K \in \Delta^{|Y|}\) be teacher distributions (treated as constants)
and let \(p_j(\theta)\in \Delta^{|Y|}\) be the student for format \(j\).
For weights \(w_i^{(j)}\ge 0\) with \(\sum_{i=1}^K w_i^{(j)}=1\), define
\[
\bar q^{(j)} \coloneqq \sum_{i=1}^K w_i^{(j)}\, q_i .
\]
Then, for each \(j\),
\begin{equation}
\label{eq:mixkl-decomp}
\begin{aligned}
\sum_{i=1}^K\! w_i^{(j)}\, \KL\!\big(q_i \,\|\, p_j(\theta)\big)
&= \KL\!\big(\bar q^{(j)} \,\|\, p_j(\theta)\big) \\
&+ \sum_{i=1}^K\! w_i^{(j)}\, \KL\!\big(q_i \,\|\, \bar q^{(j)}\big).
\end{aligned}
\end{equation}

and the last sum is constant in \(\theta\). Hence
\(\nabla_\theta\) of the left side equals
\(\nabla_\theta \KL(\bar q^{(j)}\|p_j(\theta))\).
In the uniform case \(w_i^{(j)}=\tfrac{1}{K}\), every \(p_j\) is pulled toward
the same average \(\bar q=\tfrac{1}{K}\sum_i q_i\).
\end{theorem}

\begin{proof}
Use \(\log\frac{q_i}{p_j} = \log\frac{q_i}{\bar q^{(j)}} + \log\frac{\bar q^{(j)}}{p_j}\),
sum over \(i\) with weights \(w_i^{(j)}\), and note that
\(\sum_i w_i^{(j)} q_i = \bar q^{(j)}\).
The term \(\sum_i w_i^{(j)} \KL(q_i\|\bar q^{(j)})\) contains no \(\theta\).
\end{proof}

\subsection{Datasets}
\label{sec:dataset}
We derive our validation split from the original training data (the original train is partitioned into our train and validation). Unless the original split is smaller, we hold out up to 1{,}000 examples for validation (600 for HellaSwag). For low-resource datasets, we set the derived validation size to match the size of the official validation split. When the remaining training pool is large, we cap it at 10{,}000 examples via uniform random sampling with a fixed seed (see \autoref{tab:dataset-stats}). Before stratification, we deduplicate base examples whose \emph{rendered} prompt text would otherwise repeat across splits.

All datasets used in this work are publicly available under research-oriented licenses. Specifically, RTE, CB, WSC, COPA, WiC, WinoGrande, HellaSwag, and ANLI are accessible via the Hugging Face \texttt{datasets} library~\citep{lhoest2021datasetscommunitylibrarynatural} and retain their original license terms (some permit commercial use, while others restrict use to non-commercial research). StoryCloze 2016 is available from the official ROCStories website for research use only. We do not redistribute any dataset; instead, users may obtain them directly from their original sources or through Hugging Face. Each dataset preserves its original licensing and citation requirements, and all usage in this work complies with those terms.

\begin{figure}[t]
  \centering
  \includegraphics[width=\columnwidth]{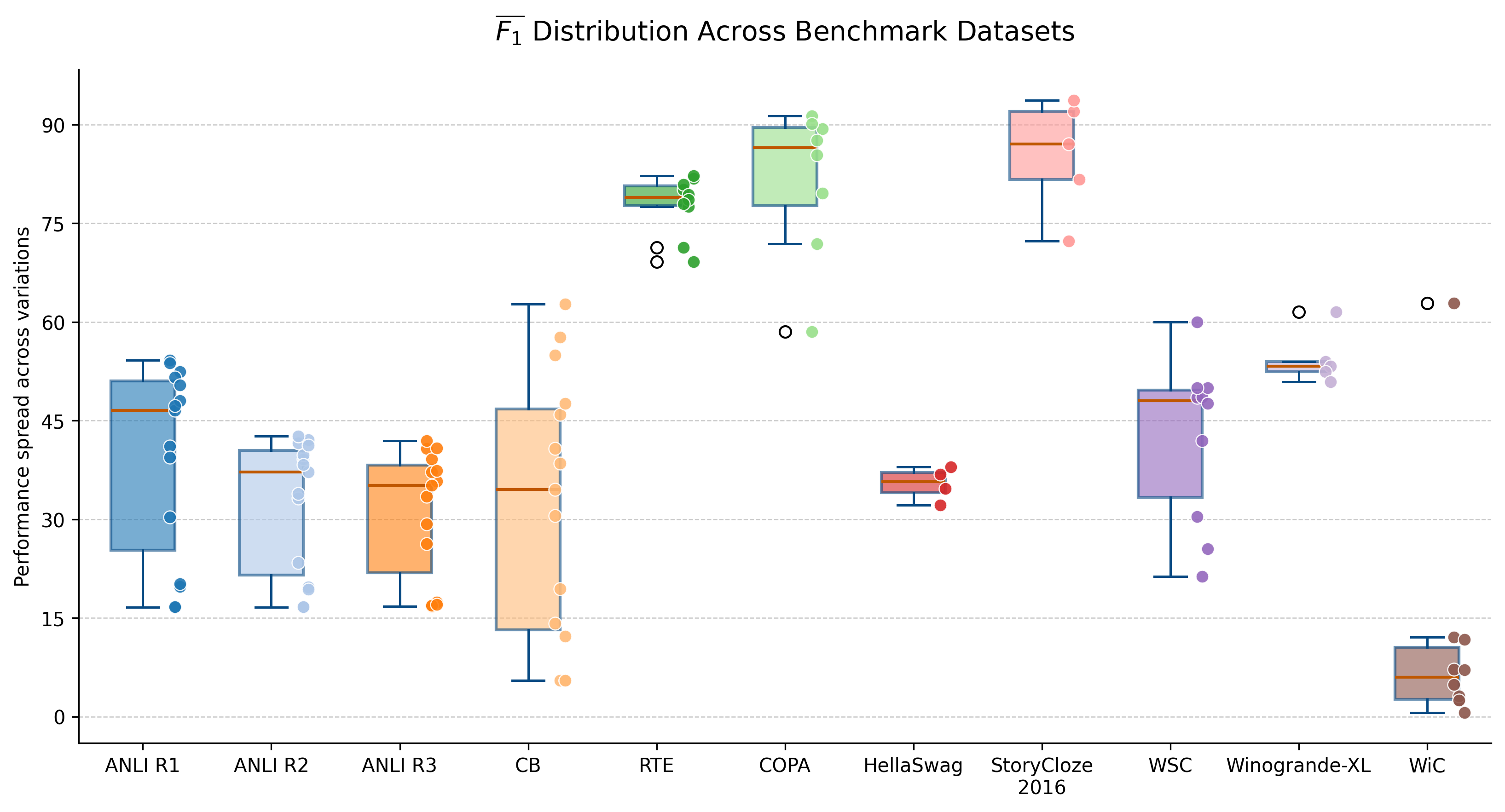}
  \caption{Per-dataset distribution of $F_1$ across prompt variations for the \texttt{Qwen2.5-3B-Instruct}.
  }
  \label{fig:app-f1-spread}
\end{figure}

\begin{table}[ht]
\centering
  \resizebox{\linewidth}{!}{%
\begin{tabular}{lrrrrr}
\toprule
Dataset & Train & Val & Test & \#Formats & \#Labels \\
\midrule
ANLI R1         & 10{,}000 & 1{,}000 & 1{,}000 & 15 & 3 \\
ANLI R2         & 10{,}000 & 1{,}000 & 1{,}000 & 15 & 3 \\
ANLI R3         & 10{,}000 & 1{,}000 & 1{,}200 & 15 & 3 \\
CB              &   194    &    56   &    56 & 15 & 3 \\
RTE             & 1{,}488  & 1{,}000 &   277 & 10 & 2 \\
COPA            &   300    &   100   &   100 &  8 & 2 \\
HellaSwag       & 10{,}000 &   600   &   600 & 4 & 4 \\
StoryCloze      &   871    & 1{,}000 & 1{,}871$^{\dagger}$ & 5 & 2 \\
WSC             &   430    &    95   &    95 & 10 & 2 \\
WinoGrande      & 10{,}000 & 1{,}000 & 1{,}267 & 5 & 2 \\
WiC             & 4{,}428  & 1{,}000 &   638 & 10 & 2 \\
\bottomrule
\end{tabular}}
\caption{``Test'' denotes the official \emph{validation} set used for evaluation because most tasks do not release test labels. \#Formats is the number of PromptSource \citep{bach2022promptsourceintegrateddevelopmentenvironment} templates used to construct each dataset's prompt variations. $^{\dagger}$StoryCloze has no public train split; we use the 2016 validation file for train/val and the 2016 test file for evaluation.}
\label{tab:dataset-stats}
\end{table}

\begin{table}[t]
  \centering
  \resizebox{\linewidth}{!}{%
  \begin{tabular}{@{}lccc@{}}
    \toprule
    Dataset & Val & Strict-majority (\%) & Pseudo-label $F_1$ \\
    \midrule
    ANLI R1    & 1{,}000 & 93.0 & 66.89 \\
    ANLI R2    & 1{,}000 & 92.6 & 69.33 \\
    ANLI R3    & 1{,}000 & 91.6 & 59.34 \\
    CB         & 56      & 78.6 & 35.82 \\
    RTE        & 1{,}000 & 97.0 & 83.10 \\
    COPA       & 100     & 97.0 & 92.78 \\
    HellaSwag  & 600     & 67.7 & 37.83 \\
    StoryCloze & 1{,}000 & 100.0 & 89.18 \\
    WSC        & 95      & 91.6 & 38.71 \\
    WinoGrande & 1{,}000 & 100.0 & 57.60 \\
    WiC        & 1{,}000 & 99.6 & 1.20 \\
    \midrule
    Aggregate & 7{,}851 & 93.9 & 59.50 \\
    \bottomrule
  \end{tabular}}
  \caption{Majority pseudo-label quality for \texttt{Qwen2.5-3B-Instruct} on validation splits. Aggregate $F_1$ is averaged across datasets by the number of strict-majority examples.}
  \label{tab:strict-majority-pseudo-labels}
\end{table}

\subsection{Implementation Details}
\label{sec:appendix-implementation-details}
\paragraph{Per-variation scoring.}
Let $y_c^{(v)}$ tokenize into $T_{i,v,c}$ answer tokens. We define the average token log-probability of choosing label $\ell_c$ under template $v$:
\begin{equation}
\mathrm{LL}_i[v,c] \;=\; \frac{1}{T_{i,v,c}}\sum_{t=1}^{T_{i,v,c}} \log p_\theta\!\bigl(y_{t}\mid x_i^{(v)}, y_{<t}\bigr),
\label{eq:avg-logprob-simple}
\end{equation}
where $y_t$ is the $t$-th answer token of $y_c^{(v)}$. These scores induce a per-variation distribution over labels:
\begin{equation}
\pi_{i,v,c} \;=\; \frac{\exp(\mathrm{LL}_i[v,c])}{\sum_{c'=1}^{L}\exp(\mathrm{LL}_i[v,c'])}.
\label{eq:per-var-softmax}
\end{equation}
The predicted label for variation $v$ is
\begin{equation}
\hat{y}_{i,v} \;=\; \arg\max_{c\in\{1,\dots,L\}} \pi_{i,v,c}.
\label{eq:app-per-var-pred}
\end{equation}

\paragraph{LoRA Configuration.}
\label{paragraph:appendix-lora-config}
We use LoRA \citep{hu2022lora} to reduce the number of trainable parameters and the compute required per step. We apply LoRA adapters in every transformer block to the attention and MLP projections, keeping backbone weights frozen. In our setup, we use rank $r{=}16$, scaling $\alpha{=}32$, and dropout $0.05$.

\paragraph{Model Selection.}
\label{paragraph:appendix-model-selection}
We select checkpoints by the highest $\overline{F_1}$ on the validation set. The $\overline{F_1}$ decreases under overfitting to particular prompts and when the majority label diverges from the gold label, and provides a robust target while $P_o$ and $\sigma_{F_1}$ quantify consistency.

\paragraph{Majority Pseudo-Label Quality.}
\label{paragraph:appendix-pseudo-label-quality}
F$^2$C uses majority agreement across prompt variations as a pseudo-labeling signal. To measure how reliable this signal is in our setting, we evaluate the strict-majority labels produced by the base \texttt{Qwen2.5-3B-Instruct} model. For each validation example, the pseudo-label is the label predicted by more than half of the prompt variations; examples without a strict majority are excluded from the $F_1$ calculation. \autoref{tab:strict-majority-pseudo-labels} reports validation-set size, strict-majority ratio, and pseudo-label macro/binary $F_1$. Overall, strict-majority labels are available for 7{,}372 of 7{,}851 examples (93.9\%). Averaging the per-dataset $F_1$ scores by the number of strict-majority examples gives 59.5.

\paragraph{F$^2$C Hyperparameters.}
\label{paragraph:appendix-f2c-hparams}
The F$^2$C hyperparameters were chosen separately for each dataset. We did not conduct an exhaustive per-dataset search. Instead, we used a small set of manual configurations and selected checkpoints by validation $\overline{F_1}$ as described above. \autoref{tab:f2c-hparams-study1} reports the values that produced the F$^2$C rows in \Cref{aggregate_all_datasets_pivoted,aggregate_all_datasets_pivoted_llama}.

\begin{table}[t]
  \centering
  \resizebox{\linewidth}{!}{%
  \begin{tabular}{@{}llccccccc@{}}
    \toprule
    Dataset & Model & $\lambda_{\mathrm{CCE}}$ & $\tau_{\text{unanimous}}$ & $k_{\max}$ & $f_{\min}$ & $f_{\max}$ & $t$ & $\beta_{\text{jsd}}$ \\
    \midrule
    ANLI R1     & Both  & 1.0 & 1.5  & 4 & 0.001 & 0.015 & 2.0 & 0.05 \\
    ANLI R2     & Both  & 1.5 & 1.2  & 4 & 0.05  & 0.30  & 2.0 & 0.30 \\
    ANLI R3     & Both  & 1.0 & 1.2  & 5 & 0.05  & 0.35  & 2.0 & 0.20 \\
    CB          & Both  & 1.0 & 9    & 7 & 0.01  & 0.10  & 2.0 & 0.01 \\
    RTE         & Both  & 0.4 & 1.4  & 3 & 0.05  & 0.35  & 2.2 & 0.35 \\
    COPA        & Both  & 1.0 & 1.0  & 2 & 0.005 & 0.05  & 3.0 & 0.05 \\
    HellaSwag   & Both  & 1.0 & 0.5  & 2 & 0.01  & 0.10  & 2.0 & 0.01 \\
    StoryCloze  & Both  & 1.0 & 1.0  & 3 & 0.01  & 0.10  & 2.0 & 0.05 \\
    WSC         & Both  & 1.0 & 2.5  & 3 & 0.20  & 0.60  & 2.0 & 1.00 \\
    WinoGrande  & Both  & 1.0 & 2.2  & 2 & 0.05  & 0.20  & 2.0 & 0.15 \\
    \multirow{2}{*}{WiC}
      & Qwen  & 1.0 & 12   & 3 & 0.01  & 0.10  & 2.0 & 0.05 \\
      & Llama & 1.0 & 0.65 & 7 & 0.10  & 1.00  & 1.0 & 0.20 \\
    \bottomrule
  \end{tabular}}
  \caption{F$^2$C hyperparameters used for the first experiment. ``Both'' means the same values were used for \texttt{Qwen2.5-3B-Instruct} and \texttt{Llama-3.2-3B-Instruct}.}
  \label{tab:f2c-hparams-study1}
\end{table}

\begin{figure*}[t]
  \centering
  \includegraphics[width=0.95\textwidth]{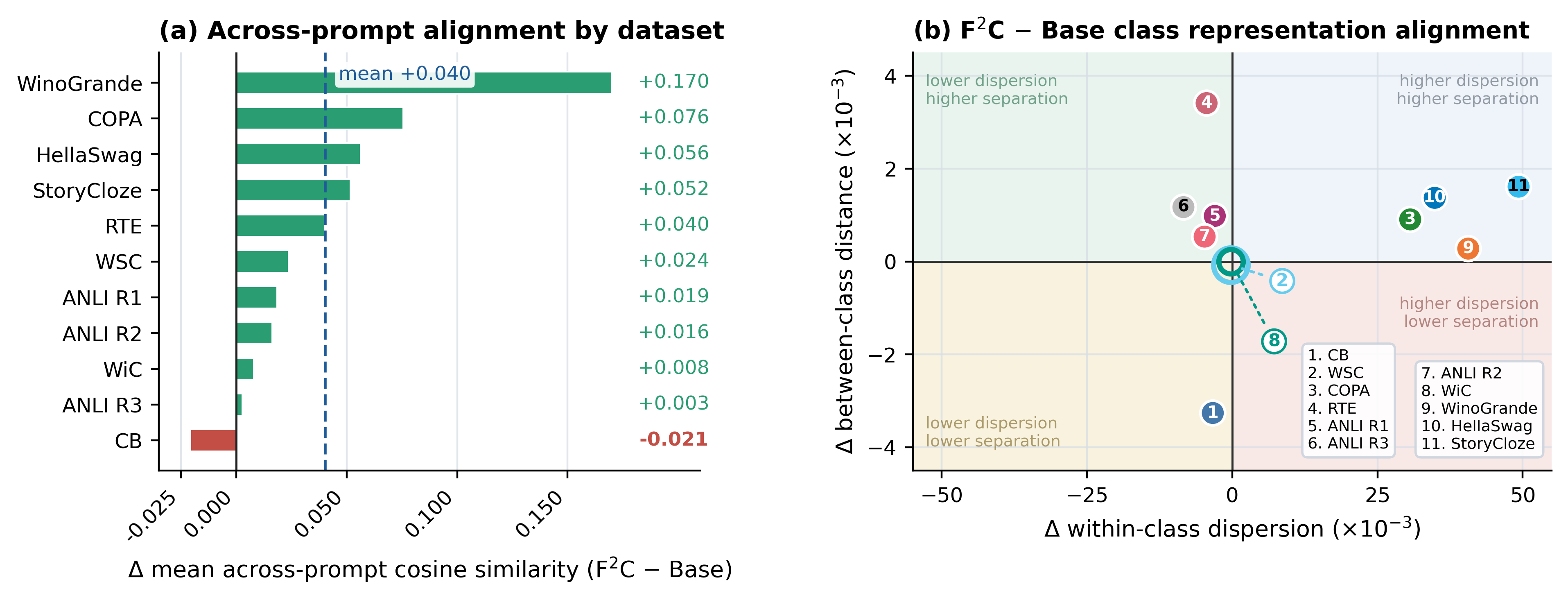}
  \caption{Representation analysis for \texttt{Qwen2.5-3B-Instruct}. In panel (a), each bar shows how much the average cosine similarity between prompt formats of the same example changes after F$^2$C. In panel (b), the x-axis shows the change in within-class distance, where negative values mean examples with the same label move closer together. The y-axis shows the change in between-class distance, where positive values mean different label groups move farther apart. Both axes show F$^2$C minus Base.}
  \label{fig:representation_alignment}
\end{figure*}

\begin{table*}[t]
  \centering
  \scriptsize
  \resizebox{\textwidth}{!}{%
  \begin{tabular}{@{}lrrrrrrrrr@{}}
    \toprule
    & \multicolumn{3}{c}{Across-prompt similarity} & \multicolumn{3}{c}{Within-class distance} & \multicolumn{3}{c}{Between-class distance} \\
    \cmidrule(lr){2-4}\cmidrule(lr){5-7}\cmidrule(l){8-10}
    Dataset & Base & F$^2$C & $\Delta$ & Base & F$^2$C & $\Delta$ & Base & F$^2$C & $\Delta$ \\
    \midrule
    ANLI R1 & 0.879 & 0.897 & \textbf{\textcolor{RoyalBlue}{+0.019}} & 0.017 & 0.014 & \textbf{\textcolor{RoyalBlue}{-0.003}} & 0.026 & 0.027 & \textbf{\textcolor{RoyalBlue}{+0.001}} \\
    ANLI R2 & 0.879 & 0.895 & \textbf{\textcolor{RoyalBlue}{+0.016}} & 0.018 & 0.013 & \textbf{\textcolor{RoyalBlue}{-0.005}} & 0.024 & 0.024 & \textbf{\textcolor{RoyalBlue}{+0.001}} \\
    ANLI R3 & 0.884 & 0.887 & \textbf{\textcolor{RoyalBlue}{+0.003}} & 0.022 & 0.014 & \textbf{\textcolor{RoyalBlue}{-0.008}} & 0.016 & 0.017 & \textbf{\textcolor{RoyalBlue}{+0.001}} \\
    CB & 0.895 & 0.874 & \textcolor{Orange}{-0.021} & 0.014 & 0.010 & \textbf{\textcolor{RoyalBlue}{-0.003}} & 0.018 & 0.015 & \textcolor{Orange}{-0.003} \\
    RTE & 0.947 & 0.988 & \textbf{\textcolor{RoyalBlue}{+0.040}} & 0.022 & 0.018 & \textbf{\textcolor{RoyalBlue}{-0.004}} & 0.028 & 0.032 & \textbf{\textcolor{RoyalBlue}{+0.003}} \\
    COPA & 0.881 & 0.957 & \textbf{\textcolor{RoyalBlue}{+0.076}} & 0.032 & 0.063 & \textcolor{Orange}{+0.031} & 0.004 & 0.005 & \textbf{\textcolor{RoyalBlue}{+0.001}} \\
    HellaSwag & 0.873 & 0.929 & \textbf{\textcolor{RoyalBlue}{+0.056}} & 0.022 & 0.057 & \textcolor{Orange}{+0.035} & 0.003 & 0.004 & \textbf{\textcolor{RoyalBlue}{+0.001}} \\
    StoryCloze & 0.933 & 0.985 & \textbf{\textcolor{RoyalBlue}{+0.052}} & 0.032 & 0.081 & \textcolor{Orange}{+0.049} & 0.004 & 0.005 & \textbf{\textcolor{RoyalBlue}{+0.002}} \\
    WSC & 0.875 & 0.899 & \textbf{\textcolor{RoyalBlue}{+0.024}} & 0.012 & 0.012 & 0.000 & 0.001 & 0.001 & 0.000 \\
    WinoGrande & 0.805 & 0.975 & \textbf{\textcolor{RoyalBlue}{+0.170}} & 0.037 & 0.077 & \textcolor{Orange}{+0.041} & 0.000 & 0.000 & 0.000 \\
    WiC & 0.851 & 0.859 & \textbf{\textcolor{RoyalBlue}{+0.008}} & 0.005 & 0.005 & 0.000 & 0.002 & 0.002 & 0.000 \\
    \bottomrule
  \end{tabular}}
  \caption{Full representation-alignment results. Across-prompt similarity is the mean cosine similarity between hidden states from different prompt formats of the same example. For the class-distance columns, each example is represented by averaging the representations of its prompt formats. $\Delta$ columns report F$^2$C minus Base; blue indicates the favorable direction (higher across-prompt cosine, lower within-class distance, or higher between-class distance), and orange indicates the opposite.}
  \label{tab:representation_alignment_summary}
\end{table*}

\paragraph{Compute, Infrastructure, and Packages.}
\label{paragraph:appendix-compute}
We fine-tune two open-weight 3B instruction-tuned models. The first experiment is conducted on both \texttt{Qwen2.5-3B-Instruct}~\citep{qwen2025qwen25technicalreport} and \texttt{Llama-3.2-3B-Instruct}~\citep{grattafiori2024llama3herdmodels}, while the out-of-domain and unseen-variation generalization experiments use \texttt{Qwen2.5-3B-Instruct}. Fine-tuning is performed using the Hugging Face \texttt{transformers}~\citep{wolf2020huggingfacestransformersstateoftheartnatural} and \texttt{peft}~\citep{peft} libraries with LoRA adapters (see~\autoref{paragraph:appendix-lora-config}).

Training was conducted on a GPU cluster running Ubuntu~22.04.5~LTS with NVIDIA Container Toolkit. Smaller datasets with fewer prompt variations were trained on a pool of seven NVIDIA RTX~A6000 GPUs (48~GB VRAM each). Larger datasets were trained on a single NVIDIA~A100~GPU (80~GB). Depending on dataset size and number of prompt formats, total training time ranged from approximately 3 to 23~hours per dataset. 

To maintain reproducibility and efficiency, we use mixed-precision (\texttt{bf16}) training, gradient accumulation, and uniform random seeds across runs. Experiments are orchestrated via custom shell scripts and Weights~\&~Biases \citep{wandb} logging for monitoring. We did not use model parallelism or distributed fine-tuning beyond single-node multi-GPU setups.

We use Python~3.9 with PyTorch~2.8.0~\citep{Ansel_PyTorch_2_Faster_2024}, \texttt{transformers}~4.56.1~\citep{wolf2020huggingfacestransformersstateoftheartnatural}, \texttt{peft}~0.17.1~\citep{peft}, \texttt{accelerate}~1.10.1~\citep{accelerate}, \texttt{datasets}~4.1.0~\citep{lhoest2021datasetscommunitylibrarynatural}, and \texttt{tokenizers}~0.22.0/\texttt{sentencepiece}~0.2.1~\citep{kudo-richardson-2018-sentencepiece, Moi_HuggingFace_s_Tokenizers_2023}.

\subsection{Representation Alignment Details}
\label{app:representation_alignment}
For each example and prompt format, we extract the final-layer hidden state immediately before the answer begins. This is the model state used to start generating the answer. For the class-level analysis, we first average prompt variations' representations of each example into one vector. We then group these example vectors by the dataset label. The within-class distance is the average cosine distance from an example vector to the mean vector of its own label group. The between-class distance is the average cosine distance between the mean vectors of different label groups. \autoref{fig:representation_alignment} visualizes the changes after F$^2$C, and \autoref{tab:representation_alignment_summary} reports the corresponding values.

\begin{figure}[h]
\centering
    \includegraphics[width=\columnwidth]{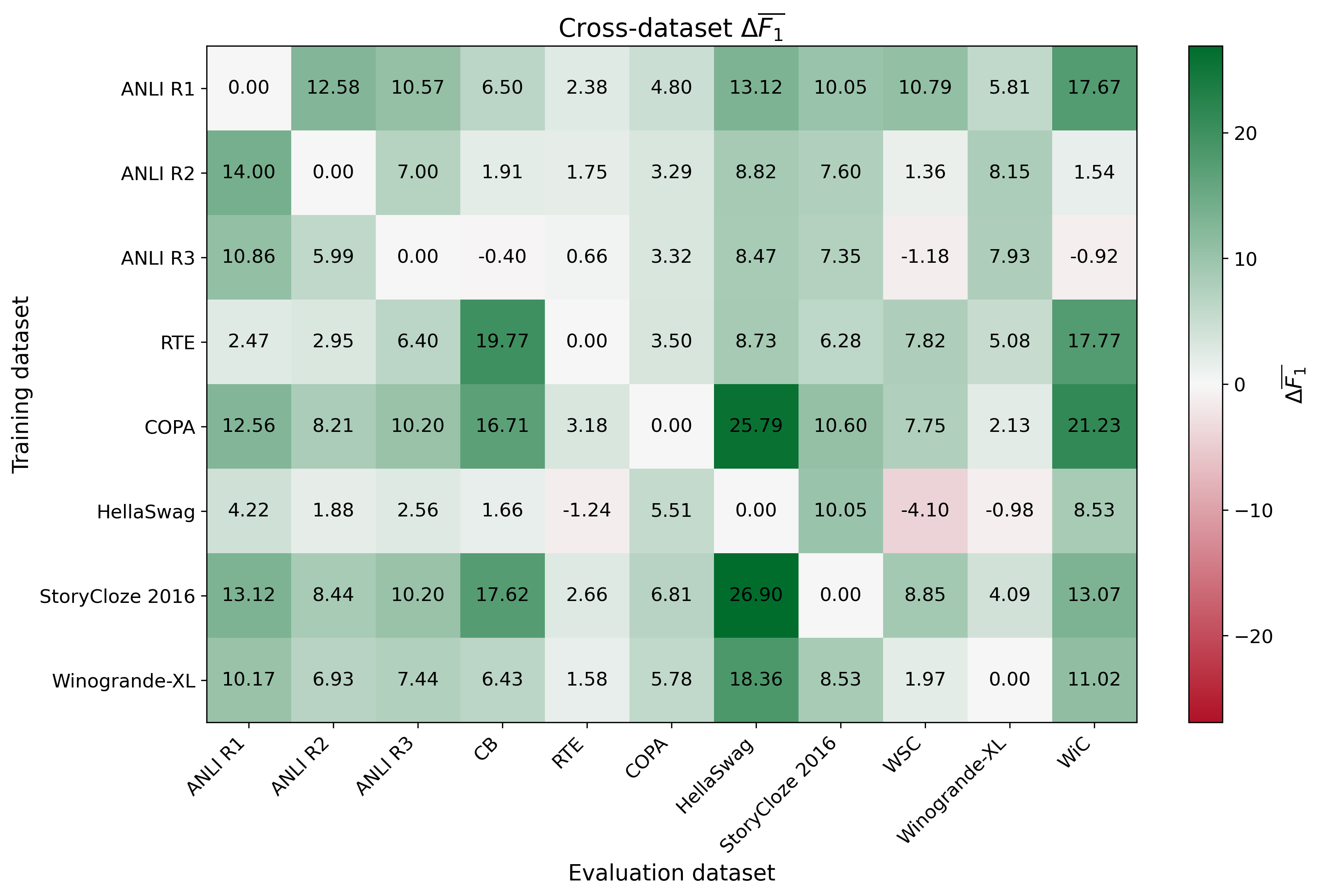}
    \caption{$\Delta\overline{F_1}$: Cross-dataset transfer under F$^2$C. Each cell shows the change relative to the base model when training on the row dataset and evaluating on the column dataset. Green indicates improvement; red indicates degradation.}
    \label{fig:hm:meanf1}
\end{figure}

\begin{figure}
  \centering
  \includegraphics[width=\columnwidth]{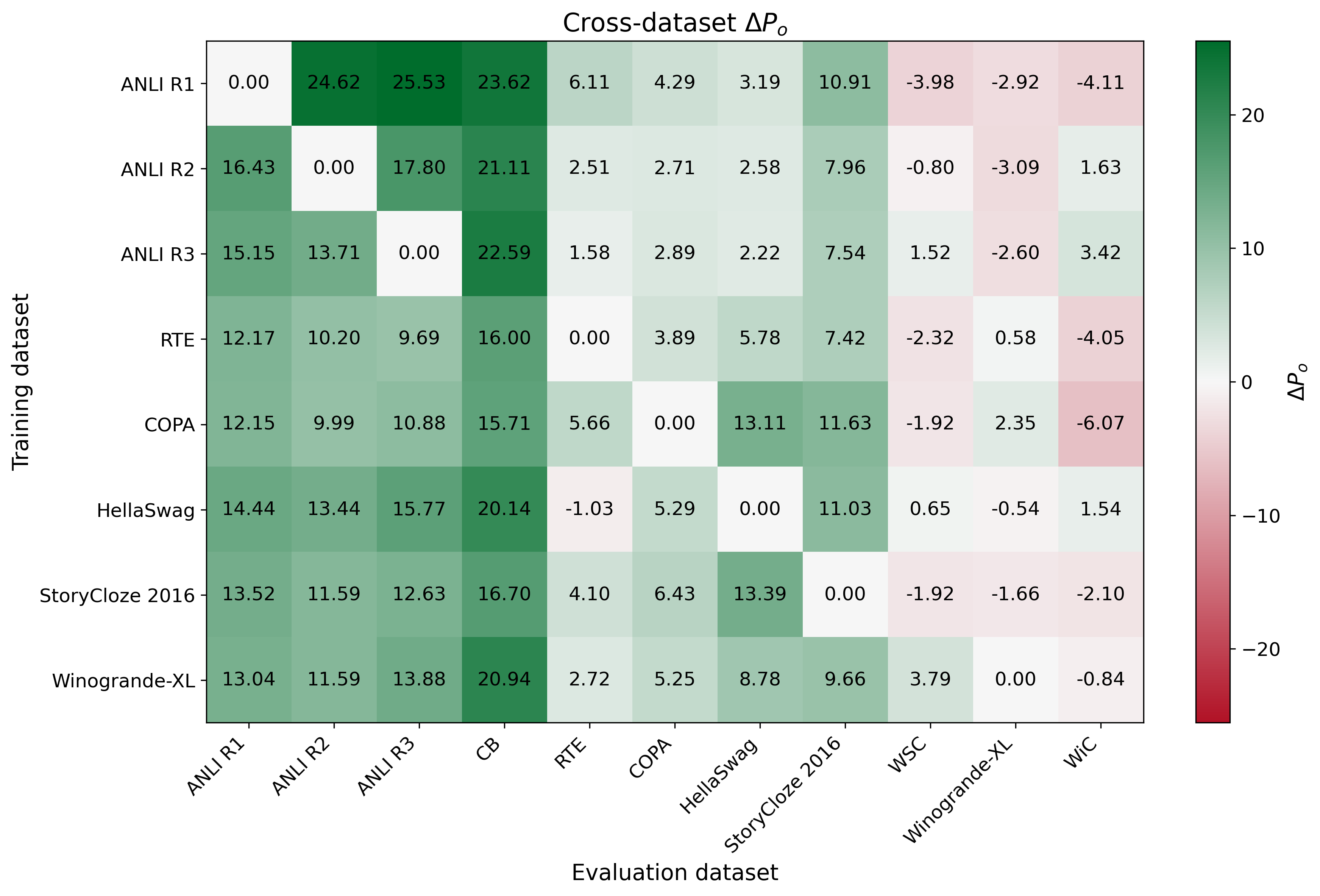}
  \caption{Cross-dataset transfer for observed agreement $P_o$. Each cell shows $\Delta P_o$ relative to the base model (train on rows, evaluate on columns). Green indicates improvement; red indicates degradation.}
  \label{fig:hm:raw_agreement}
\end{figure}

\begin{figure}
    \centering
    \includegraphics[width=\columnwidth]{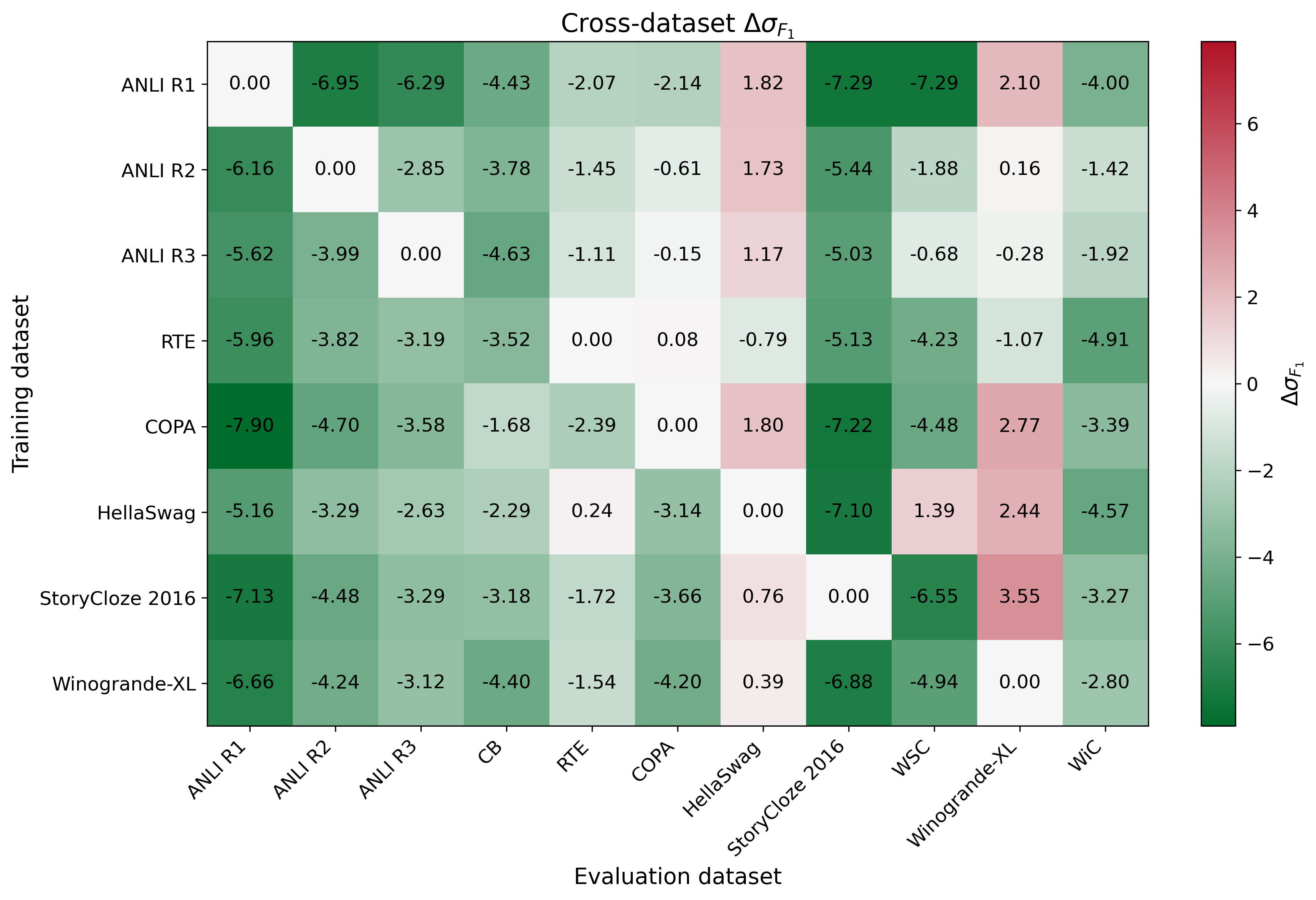}
  \caption{$\Delta\sigma_{F_1}$ (lower is better): Cross-dataset transfer under F$^2$C. Each cell shows the change relative to the base model when training on the row dataset and evaluating on the column dataset. Green indicates improvement; red indicates degradation.}
    \label{fig:hm:stdf1}
\end{figure}
\end{document}